\documentclass{article}

\usepackage{microtype}
\usepackage{graphicx}
\usepackage{subfigure}
\usepackage{booktabs} 

\usepackage{hyperref}

\usepackage[accepted]{icml2020}

\usepackage{amsmath}
\usepackage{amsthm}

\newcommand{\func}{\textsc}
\newcommand{\denom}{\delta}
\usepackage{tikz}
\usetikzlibrary{calc}

\newtheorem{theorem}{Theorem}
\newtheorem{lemma}[theorem]{Lemma}
\newtheorem{corollary}[theorem]{Corollary}

\usepackage{multirow}

\usepackage{graphicx}

\icmltitlerunning{A Pairwise Fair and Community-preserving Approach to $k$-Center Clustering}

\begin{document}

\twocolumn[
\icmltitle{A Pairwise Fair and Community-preserving Approach to $k$-Center Clustering}




\begin{icmlauthorlist}
\icmlauthor{Brian Brubach}{UMD,Wel}
\icmlauthor{Darshan Chakrabarti}{Dar}
\icmlauthor{John P.\ Dickerson}{UMD}
\icmlauthor{Samir Khuller}{NWU}
\icmlauthor{Aravind Srinivasan}{UMD}
\icmlauthor{Leonidas Tsepenekas}{UMD}
\end{icmlauthorlist}

\icmlaffiliation{UMD}{Department of Computer Science, University of Maryland, College Park, Maryland, USA}
\icmlaffiliation{NWU}{Department of Computer Science, Northwestern University, Evanston, IL, USA}
\icmlaffiliation{Dar}{School of Computer Science, Carnegie Mellon University, Pittsburgh, PA, USA}
\icmlaffiliation{Wel}{Computer Science Department, Wellesley College, Wellesley, MA, USA}

\icmlcorrespondingauthor{Brian Brubach}{bb100@wellesley.edu}

\icmlkeywords{Machine Learning, ICML}

\vskip 0.3in
]



\printAffiliationsAndNotice{}  

\begin{abstract}

Clustering is a foundational problem in machine learning with numerous applications. As machine learning increases in ubiquity as a backend for automated systems, concerns about fairness arise. Much of the current literature on fairness deals with discrimination against protected classes in supervised learning (group fairness). We define a different notion of fair clustering wherein the probability that two points (or a community of points) become separated is bounded by an increasing function of their pairwise distance (or community diameter). We capture the situation where data points represent people who gain some benefit from being clustered together. Unfairness arises when certain points are deterministically separated, either arbitrarily or by someone who intends to harm them as in the case of gerrymandering election districts. In response, we formally define two new types of fairness in the clustering setting, pairwise fairness and community preservation. To explore the practicality of our fairness goals, we devise an approach for extending existing $k$-center algorithms to satisfy these fairness constraints. Analysis of this approach proves that reasonable approximations can be achieved while maintaining fairness. In experiments, we compare the effectiveness of our approach to classical $k$-center algorithms/heuristics and explore the tradeoff between optimal clustering and fairness.

\end{abstract}

\section{Introduction}\label{sec:intro}

Clustering is one of the foundational problems in unsupervised learning and operations research. In it, we seek to partition $n$ data points into \emph{clusters} such that points within each cluster are similar according to some distance function. Its numerous applications include document/webpage similarity for search engines~\cite{Cutting1992,Zamir1997}, targeted advertising including employment opportunities~\cite{Datta15:Automated}, medical imaging~\cite{Srinivasan10:Utility,Malkomes2015}, and various other data mining and machine learning tasks. However, as machine learning has become ubiquitous, concerns have arisen about the ``fairness'' of many algorithms, especially when the data points represent human beings. In this case, we seek additional guarantees on how people will be treated beyond the typical goal of pure optimization.

The $k$-center problem is a fundamental clustering problem. The objective is to select $k$ center points and assign all other points to clusters around them such that the maximum distance from any point to its assigned center is minimized. The problem is NP-hard with the best possible approximation factor being $2$ assuming $P \neq NP$~\cite{Hochbaum1985,Gonzalez1985}. Fairness for $k$-center can have many definitions depending on the application. When the points are labeled (e.g., with racial demographics or another protected class), a \emph{group fairness} constraint may require clusters to contain a minimum amount of diversity among labels~\cite{Chierichetti2017,Bercea2018,Backurs2019}. However, we consider a different kind of fairness which bounds the probability that nearby points (presumably similar or related) are assigned to different clusters. Our approach can also address issues of discrimination against protected classes, albeit in a different way.

We introduce two new notions of fairness to the $k$-center clustering problem, \emph{pairwise fairness} and \emph{community-preserving fairness}. A $k$-center algorithm is \emph{$\alpha$-pairwise fair} if every pair of points has a probability of at most $\alpha$ of being assigned to different centers, where $\alpha(\cdot)$ is an increasing function of the distance between the two points, and $\alpha(0) = 0$. We define a \textit{community} as any subset of points with arbitrary diameter $D$ and a community is \textit{preserved} if its points are assigned to as few different clusters as possible (ideally one cluster). Communities do not need to be known or explicitly identified. An algorithm is $\beta$-community preserving if every community has probability at most $\beta$ of being partitioned into more than $t$ clusters where $\beta$ is an increasing function of the community diameter $D$ and a decreasing function of the number of clusters $t$.

The concept of pairwise fairness is relevant in settings where the points represent people and certain clusters may be preferable to others. We may assume the distance between two points represents some similarity between them and by extension, implies they should be treated similarly (assigned to the same cluster) with some related probability. We are thus being ``fair'' to each point by treating it like its nearby neighbors. The seminal work of~\cite{Dwork2012} also explores this idea of a ``\textit{fairness constraint},'' that ``similar individuals are treated similarly,'' but applied to classification and differing from our work as discussed in Section~\ref{sec:rw}.

Community preservation becomes relevant in settings where the data points are people who gain some benefit from sharing a cluster with their near neighbors. For example, consider the drawing of congressional districts and the practice of gerrymandering which has gained enormous attention and study recently. In a single-member district plurality system (e.g., the US House of Representatives), populations are partitioned into clusters called districts which each elect a single candidate based on a plurality vote. In this setting, a person or political party may draw gerrymandered districts in order to divide a community of people with shared needs, thus weakening or eliminating the power of that community to influence elections. Many cities in the United states demonstrate this phenomenon. Notably, the city of Austin, Texas is distributed among five separate congressional districts while its population is small enough to fit comfortably into two. Although it is the 11th largest city in U.S., Austin residents represent a minority in each of those five districts~\cite{AustinSize2016}.

The US Supreme Court ruled on racial gerrymandering in Thornburg v. Gingles~\cite{CaseThornburg1986}, establishing that communities of people belonging to a racial or language group should not be fractured in order to weaken their vote (subject to very specific criteria). However, partisan gerrymandering was recently ruled not justiciable by that court in Rucho v. Common Cause~\cite{CaseRucho2019}, leaving it up to the voters in individual states to advocate for some fairer approach to districting.

To combat gerrymandering, recent research has explored the use of computational approaches to draw or evaluate congressional districts~\cite{Liu2016,Altman2010,Fryer2011,Altman1998}, including $k$-clustering approaches~\cite{Cohen-Addad2018}. Like many techniques in machine learning, computational redistricting has the familiar promise of being an impartial arbiter in place of biased or adversarial human decisions. While this promise cannot be overstated, we know from the fairness literature that additional fairness constraints are often necessary. An algorithmic redistricting approach may claim to be unbiased because it does not use sensitive features such as party affiliation. However, these sensitive features may be redundantly encoded in other features as in the case of party affiliation correlating with population density in the US. Figure~\ref{fig:badcluster} gives a simple example of how a community can be deterministically separated by $k$-clustering using the $k$-center objective.

This notion of preserving communities can also be extended to problems where people are assigned to a group and benefit from having some neighbors assigned to the same group as in the problem of assigning students to public grade schools. For this problem, Ashlagi and Shi~\cite{Ashlagi2014} incorporated the concept of \textit{community cohesion}, keeping neighborhoods together. They illustrate their point by quoting Boston mayor Menino~\cite{Menino2012} saying in a 2012 State of the City address, 
"Pick any street. A dozen children probably attend a dozen different schools. Parents might not know each other; children might not play together. They can't carpool, or study for the same tests." 

Returning to the issue of protected classes, we observe that the community fragmentation imposed by current implementations of school lotteries disproportionately affects members of protected classes. On the other hand, members of more ``privileged'' classes are more likely to live in a community where assignment is not determined by lottery.

To further elaborate on the school-choice problem, we note that centers need not correspond to physical locations of schools. Many school districts, such as Boston, do not use a model wherein students are always assigned to their nearest school: e.g., a cluster could be a school bus stop for a set of students who will share a bus which is assigned to some school. 
We refer to~\cite{Ashlagi2014} for more details.

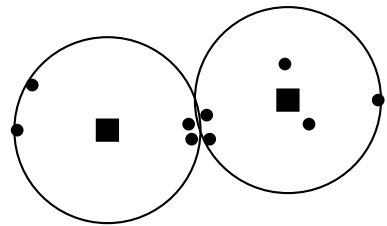
\begin{figure}[h]
	\centering
	\begin{tikzpicture}
[
	xscale=0.4,yscale=0.4,auto,thick,
	duo node/.style={ellipse,
  		inner sep=0pt, minimum width = 30, minimum height = 22, 
  		fill=white,draw
		},
  	tri node/.style={rectangle,
  		inner sep=5pt, minimum width = 30, minimum height = 22,
  		fill=white,draw
		},
  	gray node/.style={circle,
  		inner sep=0pt,minimum size=18pt, 
  		fill=black!20,draw,font=\small},  
  	white node/.style={circle,
  		inner sep=0pt,minimum size=18pt, 
  		fill=white,draw
		},
	black node/.style={circle,
  		inner sep=0pt,minimum size=4pt, 
  		fill=black,draw
		},
	center node/.style={rectangle,
  		inner sep=0pt,minimum size=8pt, 
  		fill=black,draw
		},
	font=\Large,
  	node distance=30pt 
]

\draw (3, 0) circle (88pt);
\draw (9, 1) circle (88pt);

\node[black node] at (0, 0) {};
\node[black node] at (0.5, 1.5) {};

\node[center node] at (3, 0) {};

\node[black node] at (5.7, 0.2) {};
\node[black node] at (5.8, -0.3) {};

\node[black node] at (6.3, 0.5) {};
\node[black node] at (6.4, -0.3) {};

\node[center node] at (9, 1) {};

\node[black node] at (8.9, 2.2) {};
\node[black node] at (9.7, 0.2) {};
\node[black node] at (12, 1) {};

\end{tikzpicture}
	\caption{An optimal $k$-center clustering ($k = 2$) with squares denoting the centers. This deterministically separates the community of four nearby points in the middle even though that fractured community has small diameter. 
	}
	\label{fig:badcluster}
\end{figure}

Thus, we see that pairwise fairness and community preservation have broad applications. Even in the apparently benign application of document clustering, we can view a document as its author's voice which could be negatively affected by an unfair clustering. These fairness constraints can be useful any time we wish to treat nearby points similarly, grant equal access to the strength of a community, or provide protection from efforts to weaken a community.

\subsection{Definitions and Preliminaries}

\paragraph{$k$-center clustering.} 
In the \emph{classical} (or \emph{unfair}) $k$-center problem, we are given a set $U$ of $n$ points and a parameter $k$  as input. We assume we can compute some distance function $d(u, v)$ satisfying triangle inequality on any pair of points $u, v \in U$. The objective is to choose $k$ points in $U$ to be \textit{centers} such that we minimize the maximum distance of any point in $U$ to its nearest center. In clustering, each center then defines a cluster. Typically, a point is assigned to its nearest center. However, in fair clustering and other constrained clustering variants, we may assign points to centers other than the nearest one to satisfy other goals.

\paragraph{$\alpha$-pairwise fairness.} 
We call a $k$-center algorithm $\alpha$-pairwise fair if for every pair of points $u, v \in U$, the probability that $u$ and $v$ are assigned to different centers/clusters is at most $\alpha = \alpha(u, v)$ with $\alpha(u, v)$ being an increasing function of $d(u, v)$. In this paper, we give an algorithm for the function $\alpha = d(u, v) / \denom$ where $\denom > 0$ is some distance chosen by the user. As a corollary, we focus on the natural case of $\denom = \psi R$, where $R$ is the optimal radius that can be achieved by an ``unfair'' algorithm solving the classical $k$-center problem without fairness constraints and $\psi > 0$ is a user-specified constant. The distance $R$ is used as a natural property of the problem input that can suggest what is ``reasonable'' to expect. In practice, $\delta$ could be determined by domain knowledge of a specific application. We present an algorithm that achieves $(d(u, v) / \denom)$-pairwise fairness and show that when $\alpha = (d(u, v) / (\psi R))$, the price of fairness is not too bad using both theoretical bounds and experiments.

\paragraph{$\beta$-community preserving.} 
We define a \textit{community} as any subset of points with arbitrary diameter $D$, and a community is \textit{preserved} if its points are assigned to as few different centers/clusters as possible (ideally just one cluster). In our model, communities do not need to be known or explicitly identified as part of the input. An algorithm is $\beta$-community preserving if every community has probability at most $\beta$ of being partitioned into more than $t$ clusters. Here, $\beta$ is an increasing function of the community diameter $D$ and a decreasing function of $t$. In our algorithm, every community has probability at most $\beta = (D / \denom)^t$ of being partitioned into more than $t$ clusters, $t \geq 1$, where $\denom > 0$ is some distance chosen by the user (This probability is a decreasing function of $t$ since we may assume $D/\denom < 1$: if $D/\denom \geq 1$, then the probability is trivially at most $1$). As with pairwise fairness, we examine the natural choice of $\denom = \psi R$. Here, we show that we can give the guarantee that every community has a probability of at most $\beta = (D / (\psi R))^t$ of being partitioned into more than $t$ clusters. We include $t$ because it captures how fragmented a community becomes more than simply whether or not it has been separated.

\paragraph{Randomization.} 
Both definitions of fairness assume a randomized algorithm and the probabilities discussed are over the randomness in the algorithm. As with some other fairness problems (e.g., fair allocation of indivisible goods), randomness is essentially required to achieve meaningful gains in fairness. Otherwise, it is easy to construct worst case examples where a fair deterministic algorithm must place all points in one large cluster while a fair randomized algorithm could achieve results close to the unfair optimal. 
Randomization can even be necessary to meet certain fairness criteria such as the right to a chance to vote in a district with voter distribution similar to a randomly sampled legal district map~\cite{Brubach2020Meddling}. We further note that our pairwise fairness definition makes no assumption of independence or correlation between the separation probabilities of different pairs of points. It is an individual guarantee for each pair of points. Consideration of multiple points at once is addressed by the community preservation definition.

\paragraph{Focus on $\denom$ as a function of optimal unfair radius $R$.} 
We consider the special case of $\delta$ depending on $R$ in our analysis because $R$ is a reasonable threshold of nearness related to the properties of a given dataset and the $k$-clustering task at hand. For example, if a community is geographically larger than the optimal unfair clusters themselves, it may be reasonable to partition this community into multiple clusters whereas a small community which can fit easily into a cluster should have some chance of being preserved.

\paragraph{Approximation ratio and price of fairness.} 
The \emph{approximation ratio} of an algorithm for an NP-hard minimization problem like $k$-center is typically defined as a bound on the ratio of the algorithm's solution to the solution of an optimal algorithm. The \emph{price of fairness} for a fair variant of a problem is the ratio of the best solution for the fair problem to the best solution for the unfair problem. In our case, the best benchmark we are able to compare our fair algorithm to is the optimal unfair $k$-center solution. Thus, our approximation ratios simultaneously show a bound on the price of fairness for our proposed fairness definitions. This price of fairness can affect the choice to use a fair algorithm for both practical and legal reasons. From a legal perspective, the disparate impact of an unfair algorithm can be permitted due ``business necessity'' if the added cost of fairness is too burdensome~\cite{CivilRights1991,SCOTUS15:Texas}, but a low price of fairness could potentially preclude this defense.

\subsection{Related Work}
\label{sec:rw}

There is a long line of work on the classical $k$-center problem. A 2-approximation is known and is the best possible assuming $P \neq NP$ \cite{Hochbaum1985,Gonzalez1985,Hochbaum1986}. Followup work has studied many variations of the problem including capacitated~\cite{Khuller1996,Fernandes2018}, connected~\cite{Ge2008}, fault tolerant~\cite{Khuller2000,Fernandes2018}, with outliers~\cite{McCutchen2008,Chakrabarty2016,Malkomes2015}, and minimum coverage~\cite{Lim2004}. Other settings include streaming~\cite{Charikar1997,Charikar2003,McCutchen2008}, sparse graphs~\cite{Thorup2001}, and distributed algorithms for massive data~\cite{Malkomes2015}. However, our formulation of pairwise fairness and community preservation, has not been studied.

On the fairness side, our notion of pairwise fairness is partially inspired by~\cite{Dwork2012}. That work focused on binary classification as opposed to clustering and used techniques from differential privacy to achieve fairness guarantees. More specifically, they assume access to a separate similarity metric on the data points and require similar points to have similar distributions on outcomes. While our model is related, it differs in two crucial ways. First, we do not use (or require) a separate similarity metric. The similarity of two points is defined by the same metric space we are clustering in. Second, we bound the probability that two points are actually \emph{assigned} to the same cluster rather than having similar distributions. This is important for applications in which nearby points derive a benefit from being clustered together or when the meaning of a cluster is not defined prior to the realization of assignments.

For $k$-center specifically,~\cite{Chierichetti2017} considered an entirely different ``balance'' constraint definition of fairness (aka group fairness) wherein each point is given one of two possible labels and each cluster should contain a minimum percent representation of each label. Follow-on work expands their model~\cite{Rosner18:Privacy,Bera19:Fair} and addresses concerns in privacy while~\cite{Kleindessner19:Guarantees} applied their definition of fairness to spectral clustering. Additional work improved scalability~\cite{Backurs2019} and improved approximation ratios while allowing an unfair solution to be transformed into a fair one~\cite{Bercea2018}. Separately, and motivated by the bias mitigation in data summarization,~\cite{Kleindessner19:Fair} also looks at a different form of $k$-center fairness. Zemel et al.~\cite{Zemel2013} address fairness in classification 
by first transforming the input data into an intermediate representation that balances goodness of representation with removal of certain traits before classification is performed. This first step is a form of clustering with fairness concerns. 
Finally, there are fair service guarantees for individuals that bound the distance from each point to its nearest center (or facility)~\cite{Harris2019,Jung2020,Mahabadi2020}.

Regarding community preservation, \cite{Ashlagi2014} observed that assigning students to schools via an independent lottery mechanism fractures communities by sending neighboring students to different schools. They proposed a correlated lottery algorithm that that maintains the same expected outcomes for individual students while preserving ``community cohesion.'' We note that they define communities by partitioning a city into a grid with each square representing a community, whereas we allow any bounded diameter subset of points to be a community.

Bounding the probability of separating nearby points and similar negative-binomial-type (or discrete exponential) distributions have been used in numerous other settings. Some examples include locality sensitive hashing (LSH)~\cite{Indyk1998LSH,Gionis1999LSH,Datar2004LSH}, randomly shifted grids~\cite{Hochbaum1985Grids}, low diameter graph decompositions~\cite{Linial1993}, and randomized tree embeddings~\cite{Bartal1996,FRT2003}. Our work differs from this past work in the modeling of fairness applications and the challenge of balancing fairness with the k-center objective which is not guaranteed in something like LSH. More commonly, an approach like LSH is used to speed up and scale clustering algorithms with approximate near neighbor search or partitioning data for parallel and distributed algorithms.

\subsection{Our Contributions}\label{sec:contributions}

In addition to presenting new definitions of fairness in clustering, we show how any algorithm for the $k$-center problem can be extended to ensure $\alpha$-pairwise fairness and $\beta$-community preservation at the expense of a $\log k$ approximation factor (also price of fairness). We bound our fair algorithm in comparison to the optimal radius achieved in the ``unfair'' classical $k$-center problem. There are two reasons for this. One is that the ``unfair'' optimal serves as the best known lower bound to the fair optimal. The other is that it captures the \textit{price of fairness}. In other words, it upper bounds the price we must pay in expanding the radius in order to achieve our fairness objectives.

\begin{theorem}
\label{thm:pair}
There exists an algorithm which finds an $O(\log{k})$-approximation to the $k$-center problem (i.e., the maximum cluster radius is at most $O(R \log k)$) with high probability and such that every pair of points $u$ and $v$ is separated with probability at most $\alpha =  d(u,v) / (\psi R)$, where $R$ is the maximum radius obtained by any chosen $k$-center algorithm and $\psi > 0$ is a user-specified constant.
\end{theorem}

The community preserving property in Corollary~\ref{cor:com} follows from the pairwise guarantee. A strength of this formulation is that we \emph{do not need to explicitly identify communities} in the data to preserve them with nontrivial probability.

\begin{corollary}
\label{cor:com}
There is an efficient $O(\log{k})$-approximation algorithm for $k$-center 
(i.e., the maximum cluster-radius is at most $O(R \log k)$) with high probability and such that every subset of points with diameter $D$ is partitioned into more than $t$ separate clusters, for any $t \geq 1$, with probability at most $\beta = (D / (\psi R))^t$ where $R$ is the maximum radius obtained by any chosen $k$-center algorithm. Here, $\psi > 0$ is a user-specified constant.
\end{corollary}

For both Theorem~\ref{thm:pair} and Corollary~\ref{cor:com}, we note that for some pairs of points (or communities) the value of $\alpha$ (or $\beta$) may be greater than $1$ and therefore not a valid probability. For these cases, the bound on fairness is trivially true. 
The constant factors in our big-Oh notation also depend on the constant $\psi$ and our experiments in Section~\ref{sec:exp} show that there are not large hidden constants in practice.

Beyond theoretical results, we further explore the algorithm experimentally in Section~\ref{sec:exp} on 40 different problem instances of a benchmark dataset  
to show that it performs as expected or better. On the benchmark problems, we illustrate in Figure~\ref{plot:pmed1-40} how tuning a parameter in our algorithm can adjust the trade-off between fairness and minimizing the cluster radius. 
In Section~\ref{sec:real-exp},  we evaluate our algorithm on a real dataset over different target numbers of clusters. The results suggest that our fair approach is not only more fair, but more consistent in its fairness as $k$ varies when compared to a standard ``unfair'' algorithm. Thus, we can remove the ability of a bad actor to cause unfairness by adjusting the number of clusters $k$.

While our theoretical and experimental analysis focuses on approximating the radius and fairness, we note that the running time of our proposed algorithm is dependent primarily on the algorithm/heuristic for the initial clustering. Our reassignment algorithm is rather fast with a running time of $O(kn)$. In practice, the running time is dominated by the initial clustering rather than our reassignment algorithm.

\section{The Fair Algorithm}
\label{sec:fairextension}

We show how to extend any $k$-center algorithm to guarantee pairwise fairness at the expense of a larger approximation factor. The idea is to first run an ``unfair'' $k$-center algorithm and order the clusters arbitrarily. Then, one-by-one, we expand the radius of each cluster by a value sampled independently from an exponential distribution. Any point which falls within the radii of more than one of these expanded clusters is assigned to the earliest one in the ordering. 

We use $C_i$ to refer to the $i^\text{th}$ cluster found by the initial ``unfair'' algorithm and $c_i$ to refer to its center. Similarly, we use $C'_i$ to refer to the $i^\text{th}$ expanded cluster that we will finally output and $c'_i$ to refer to its center. For readability, we also refer to $C_i$ and $c_i$ as original and $C'_i$ and $c'_i$ as final. Let $R_i = \max_{u \in C_i} d(c_i, u)$ be the radius of $C_i$ and $R = \max_{i} R_i$ be the maximum radius of any cluster found by the original clustering step. Let $\psi$ be any chosen constant greater than $0$. The approach is summarized in Algorithm~\ref{alg:fair}.

We note that in the for loop of steps 4 to 6 of Algorithm~\ref{alg:fair}, the centers 1 through $k$ are processed in an arbitrary order. Because of this, our proofs also hold if the center are processed in a random order or some particular order aligned with another side objective.

\begin{algorithm}[tb]
   \caption{\func{FairAlg}}
   \label{alg:fair}
\begin{algorithmic}
   \STATE {\bfseries Step 1:} Run any chosen $k$-center algorithm and order the clusters arbitrarily from $1$ to $k$. Let $R$ be the maximum distance of any point to its center.
   \STATE {\bfseries Step 2:} Let $C_i$ be a set of points denoting cluster $i$. Let $c_i \in C_i$ be the center of $C_i$ and $R_i$ be the radius of $C_i$.
   \STATE {\bfseries Step 3:} Treat all points including centers as ``unclustered'' and construct a new set of clusters denoted $C'_i$.
   
   \FOR{$i = 1$ {\bfseries to} $k$}
   
   	\STATE {\bfseries 4:} Sample an independent random variable $x_i$ from an exponential distribution with parameter $\lambda = 1/(\psi R)$. Let $X_i$ be the realization of that random variable.

	\STATE {\bfseries 5:} Construct cluster $C'_i$ by adding every unclustered point within radius $R_i + X_i$ from original center $c_i$.

	\STATE {\bfseries 6:} If $c_i$ was unclustered at the start of this iteration designate it as the center $c'_i$ of $C'_i$. Otherwise, if $c_i$ has been added to a previous cluster $C_j$, $j < i$, then choose any other previously unclustered point in $C'_i$ to be the center $c'_i$. If no such point exists, call the cluster empty.
 
   \ENDFOR
\end{algorithmic}
\end{algorithm}

We first prove that Algorithm~\ref{alg:fair} achieves $\alpha$-pairwise fairness for $\alpha = d(u,v) / (\psi R)$. At a high level, the memoryless property of exponentially distributed random variables allows our algorithm to achieve the guarantee in Lemma~\ref{lem:sep}.

\begin{lemma}
\label{lem:sep}
For any pair of points $u$ and $v$ with distance $d(u,v)$, the probability that Algorithm~\ref{alg:fair} separates $u$ and $v$ into two separate clusters is at most $d(u,v) / (\psi R)$ where $R$ is the maximum radius obtained by the initial algorithm used in step 1 and $\psi > 0$ is a user-specified constant.
\end{lemma}
\begin{proof}

For an arbitrary pair of points $u, v \in U$, consider the first iteration $i$ in which at least one of the points is added to a final cluster $C'_i$. Without loss of generality, let $u$ be the closer point to the original center $c_i$ and note that $d(c_i, v) - d(c_i, u) \leq d(u, v)$ due to triangle inequality. If $d(c_i, v) < R_i$, both points will be added to $C'_i$ regardless of the value of $X_i$ and the probability of separating them is $0$. Otherwise, the probability of separating them is the probability that the value $R_i + X_i$ falls between $\max(d(c_i, u), R_i)$ and $d(c_i, v)$ given that $R_i + X_i > d(c_i, u)$.
\begin{align*}
\Pr&[\text{$u$ and $v$ are separated by $C'_i$} \, | \, R_i + X_i > d(c_i, u)] \\
&\leq 1 - e^{-\lambda d(u,v)} \label{E-1} 
= 1 - e^{-d(u,v) / \psi R} \\
&\leq \frac{d(u,v)}{\psi R}
\end{align*}
\end{proof}

We now bound the amount that the radius of any cluster will increase beyond the maximum value $R$ achieved by the original ``unfair'' algorithm from step 1 of Algorithm~\ref{alg:fair}.

\begin{lemma}
\label{lem:approx}
The maximum radius of a cluster found by Algorithm~\ref{alg:fair} is $O(R \log{k})$ with high probability. 
\end{lemma}
\begin{proof}

We start by upper bounding the probability that any cluster $C'_i$ contains a point at distance greater than $O(R \log{k})$ from the original center $c_i$ of $C_i$. This will suffice to prove the lemma for the clusters where $c'_i = c_i$.
\begin{align*}
\Pr[\exists X_i > R \log{k}] 
&\leq k \Pr[X_i > R \log{k}] \\
&= k e^{-\lambda R \log k} 
= k e^{-\log k / \psi} \\
&= k^{1 - 1 / \psi}
\end{align*}
Now, suppose $c_i$ was added to some cluster $C'_j$, $j < i$, and could not be chosen as the final center of $C'_i$. Then the chosen center $c'_i$ of $C'_i$ must be at most $R \log{k}$ distance from $c_i$ with high probability by the above bound and the fact that $X_i$ and $X_j$ were sampled independently. Thus, by triangle inequality, the radius of such a cluster would be at most $2 R \log{k} = O(R \log{k})$ with high probability.
\end{proof}

Lemma~\ref{lem:com} extends Lemma~\ref{lem:sep} to community preservation.

\begin{lemma}
\label{lem:com}
For any subset of points $S$ with diameter $D$, the probability that Algorithm~\ref{alg:fair} partitions $S$ into more than $t$ separate clusters, $t \geq 1$, is at most $(D / (\psi R))^t$ where $R$ is the maximum radius obtained by the initial algorithm used in step 1 and $\psi > 0$ is a user-specified constant.
\end{lemma}
\begin{proof}

To bound the probability of the number of final clusters $S$ is partitioned into, let $j$ be the index of the last cluster to recruit a member of $S$. Let $C^S$ be the set of clusters where some $w \in S$ has $d(c_i, w) \leq R_i + X_i $ and $i \leq j$. In other words, $C^S$ contains the only clusters which could possibly separate $S$. 
We observe that the final number of clusters is upper bounded by the number of clusters in $C^S$ whose radii around original center $c_i$ separates $S$ regardless of whether the cluster $C'_i$ was actually able to recruit any unclustered points from $S$. We note that such a separation can increase the number of partitions by at most one. 

By the same arguments as in the proof of Lemma~\ref{lem:sep}, given that at least one point $w \in S$ has $d(c_i, w) \leq R_i + X_i $, the probability that the radius around original center $c_i$ separates $S$ is at most $d / \psi R$. This follows from taking $u$ and $v$ to be the points in $S_i$ which are closest and farthest, respectively, from the center and upper bounding $d(c_i, v) - d(c_i, u) \leq d(u, v) \leq d$. 
We further note that if any $C'_i \in C^S$ fails to separate $S$, then any unassigned points in $S$ will be assigned to $C'_i$ and no future clusters will be able to separate $S$. Thus, for $S$ to be split into more than $t$ clusters, the first $t$ clusters in $C^S$ must each separate $S$. This occurs independently with probability at most $d / \psi R$ for each cluster after conditioning on the clusters' membership in $C^S$.
\end{proof}

\section{Benchmark dataset experiments}
\label{sec:exp}

We ran experiments on the well-known $p$-median dataset from OR-Lib~\cite{OR-Lib1990} which contains 40 different problem instances. It was originally generated for the $p$-median problem~\cite{Beasley1985}, but has since been commonly used to evaluate $k$-center algorithms and heuristics~\cite{Mihelic2005,Garcia-Diaz2017}. Another advantage to benchmarking with this data is that the optimal radius is now known for each of the 40 problem instances in the dataset. The specified number of centers, $k$, varies across the instances with the smallest being $k = 5$ and the largest being $k = 200$. We evaluate our approach on all 40 problem instances.

\begin{figure*}[h!]
\includegraphics[width=9cm]{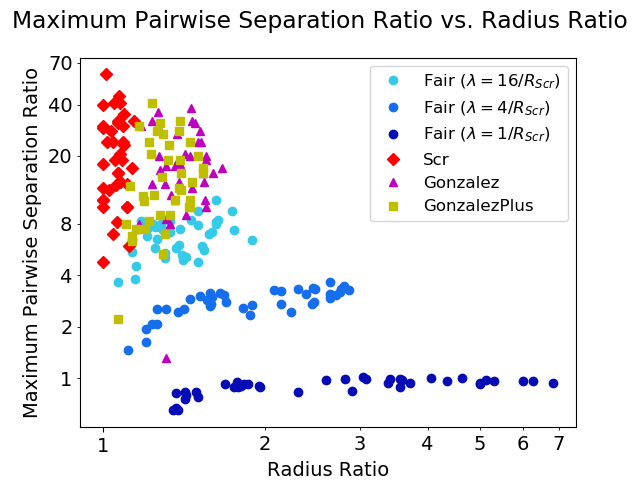}
\includegraphics[width=9cm]{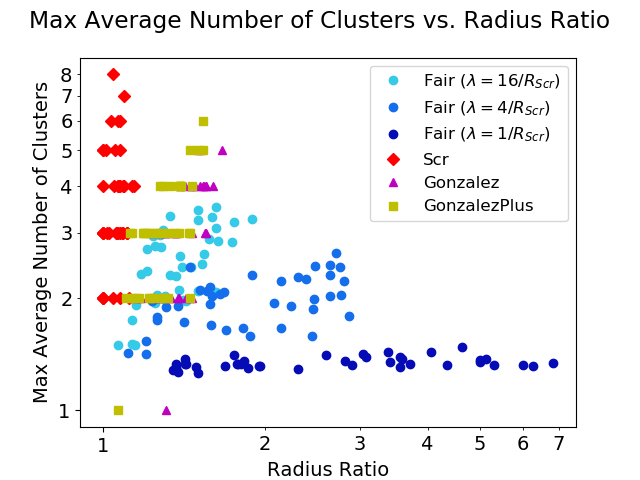}
  \caption{Comparison across all 40 instances of the pmed dataset. The three shades of blue circles show our algorithm parameterized by $\lambda$ of $16/R_{Scr}$, $4/R_{Scr}$, and $1/R_{Scr}$, while other shapes show the unfair algorithms. Points closer to the bottom are more fair while points closer to the left represent solutions with a smaller radius. Our algorithm outperforms the unfair algorithms in both separation ratio (left) and community preservation (right) at the expense of radius as expected. Comparing the three versions of the fair algorithm, we see a clear trade-off between fairness and minimizing the radius.
  }
\label{plot:pmed1-40}
\end{figure*}

\subsection{Experiment design}
\label{sec:expdesign}

We compare three ``unfair'' algorithms to multiple versions of our fair algorithm using different parameters. In all cases, we use $d(u,v)/R_{Scr}$ as the target separation probability bound where $R_{Scr}$ is the radius found by Scr heuristic defined below. This choice is somewhat arbitrary, but it provides a fixed target to compare the different algorithms and the Scr radius serves as a fairly close approximation to unfair optimal, which we assume is unknown to the algorithms. 
Thus, if someone were to apply our algorithm in practice, the radius found by Scr (or other chosen heuristic) would be their best guess at the optimal radius. 
Each of the three deterministic ``unfair'' algorithms was run once per dataset, while each fair algorithm was run for $10{,}000$ trials in order to evaluate average performance.

\paragraph{The ``unfair'' algorithms.} 
In order to compare and evaluate our algorithm, we implemented three algorithms for the classical $k$-center problem: Gonz1, Gonz+, and Scr. The first two are variations of the famous Gonzalez algorithm~\cite{Gonzalez1985}. While they do not achieve the strongest results on this dataset, they give theoretically optimal approximations and are known for their exceptional speed and simplicity. The third algorithm, Scr, achieves nearly optimal results~\cite{Mihelic2005} on the dataset. Recent heuristics have yielded marginal improvements over Scr~\cite{Garcia-Diaz2017}, but we choose Scr because it achieves nearly the same results while remaining fairly simple to implement and reproduce.

\paragraph{Fair algorithm implementation.} 
Our implementation of the fair algorithm uses Scr to find the initial set of centers. We choose Scr since it gets the tightest radius to begin with. We parameterize our algorithm with the mean, $1/\lambda$, of the exponential distribution we sample from, where $\lambda$ is the exponential parameter used in Algorithm~\ref{alg:fair}. For our ``Exact'' fair algorithm we set $\lambda = 1/R_{Scr}$ which corresponds to a theoretical separation ratio at most $d(u,v)/R_{Scr}$ for each pair of points $(u, v)$. For our ``Medium'' fair algorithm, we set $\lambda = 4/R_{Scr}$ since $R_{Scr}/4$ is our target community radius described in our comparison criteria below. Finally, for our ``Tight'' fair algorithm, we simply divide our mean by another factor of $4$ to get $\lambda = 16/R_{Scr}$. Using three different parameters gives some indication of the compromise that can be reached between minimizing the radius and optimizing the fairness.

In addition, our implementation makes two natural modifications to Algorithm~\ref{alg:fair} that do not affect the theoretical bounds. First, the list of centers found in Step 1 is uniformly randomly permuted before growing the clusters. Second, if we have to choose a new center point in Step 6, we choose the point in the cluster which minimizes the radius as opposed to any arbitrary point.

\paragraph{Comparison criteria.} 
We compared the algorithms in terms of three criteria: radius, pairwise fairness, and community preservation. First, we looked at the approximation of the radius with respect to the unfair optimal. This is the ratio of the radius found by each algorithm to the optimal radius (known for this dataset due to~\cite{Daskin2000,Elloumi2004,Mladenovic2003}). For the randomized algorithms, we give the average radius across all trials. 
More specifically, this is an average taken over the max radius of each trial derived from the cluster with the largest radius in keeping with the $k$-center objective.

To evaluate the pairwise fairness, we considered only pairs of points with $d(u,v) \leq R_{Scr}$ (i.e. target maximum separation probability at most $1$). For each such pair, we compute the ratio of the algorithm's separation probability to the target maximum separation probability. 
For the deterministic algorithms, the numerator of this ratio is $0$ (not separated) or $1$ (separated). For the randomized algorithms, the separation probability is given as the number of trials where the points were separated divided by the total number of trials ($10{,}000$). Then, for each algorithm, we take the worst separation probability ratio among all pairs of points with distance at most $R_{Scr}$. For the deterministic algorithms this is determined by the nearest pair of points which is separated.

In order to address communities, we needed to define some specific type of community since analyzing every possible subset of points is infeasible. In practical applications there may be some specific target communities based on domain information. However, for this experiment we say that every point defines a community including itself and all other points within a distance of at most $R_{Scr}/4$ from it. 
In practical terms, each point could be a person and its community could be that person's neighborhood. We assume the community radius is smaller than the clustering radius as is the case with real world examples such as congressional voting districts. 
For each point's community, we count the number of different clusters its points have been assigned to. To show the worst case, we highlight the most fractured community, meaning the community split into the most different clusters. For the randomized algorithms, each community gets an average value over all trials and we note the community with the worst average.

\begin{figure*}[h!]
\begin{center}
\includegraphics[width=7cm]{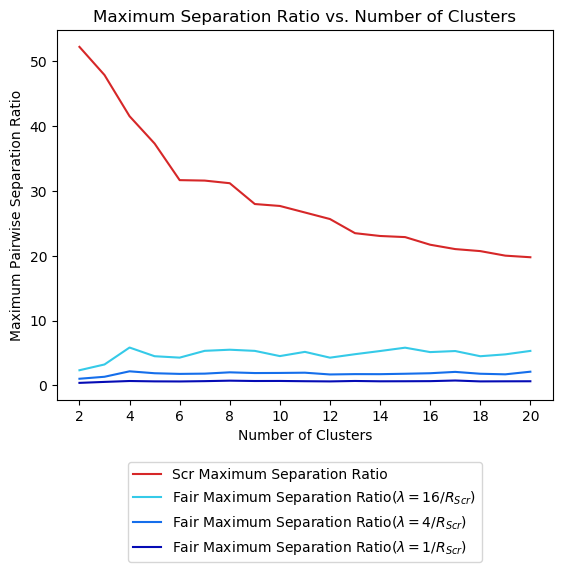}
~~~~~~~~
\includegraphics[width=7cm]{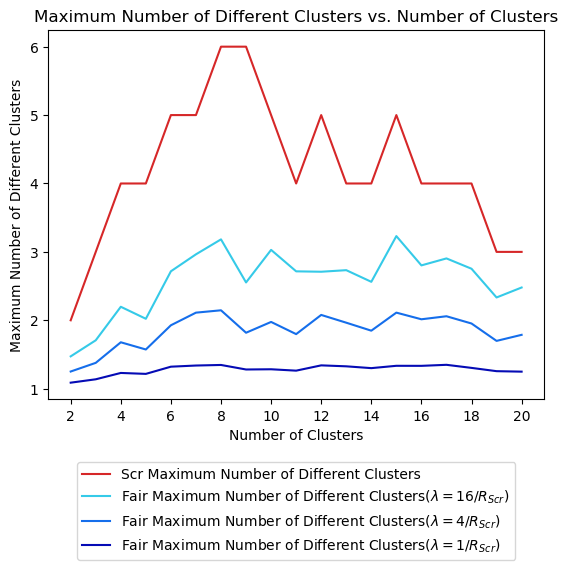}
\end{center}
  \caption{Comparison over different numbers of clusters, $k$, from 2 to 20 on the adult dataset. We measure the maximum pairwise separation ratio (left) and maximum number of different clusters any community is separated into (right). In both cases, lower values on the y-axis are more fair. We compare Scr to three versions of our algorithm parameterized by $\lambda$ of $16/R_{Scr}$, $4/R_{Scr}$, and $1/R_{Scr}$. We see that the most extreme fair algorithm, $\lambda = 1/R_{Scr}$, is not only the most fair, but most consistent across different values of $k$.
  }
\label{plot:adult-fair}
\end{figure*}

\subsection{Experimental results}

Figure~\ref{plot:pmed1-40} summarizes the main results of our $k$-center benchmark dataset experiments. Overall, we see a clear trade-off between fairness and minimizing the radius with the three different parameters of our fair algorithm.

For the maximum pairwise separation ratio, even our Tight algorithm is more fair than any of the unfair algorithms across almost all instances without paying too much cost in terms of larger cluster radii. This implies that even slight random perturbation of the clusters can dramatically improve fairness with limited impact on the maximum radius of the solution. The pairwise separation ratios for the Exact fair algorithm are roughly $1$ or less. Some pairwise separation ratios slightly above $1$ are to be expected even for Exact since this is the worst performance of any pair of points in a given problem instance and we are running only $10{,}000$ trials of each randomized algorithm. Likewise, the pairwise separation ratios of the Medium fair algorithm are roughly upper bounded by $4$ as expected. In several cases, the pairwise separation ratio for Exact is actually below $1$ meaning that \emph{every} pair of points $(u, v)$ in those instances with $d(u, v) \leq R_{Scr}$ is separated with probability less than $d(u, v)/R_{Scr}$.

With respect to community preservation, we can see that the performance of Tight approaches the two Gonzalez algorithms and is only slightly fairer than the unfair algorithms. However, the maximum average number of different clusters for Exact is always less than two. On some instances, Scr separates some small community of nearby points into 6 or more clusters while Exact gives every community a guarantee that it will be preserved in a single cluster with fairly good probability.

In summary, the fair and unfair algorithms perform as expected yielding a reasonable trade-off between fairness and small radii. The effect of adjusting the $\lambda$ parameter varies based on the structure of the input. In many cases, using a smaller $\lambda$ than Exact could be a desirable heuristic if assumptions can be made about the input. Another option, time permitting, is to perform a binary search for the $\lambda$ which best satisfies a desired balance of fairness and cluster tightness.

\section{Experiments on real data}
\label{sec:real-exp}

We ran additional experiments on a sample of 1,000 points from the adult dataset~\cite{Kohavi1996}. To create the metric space, we normalized the numeric features of age, education-num, and hours-per-week and used them to define points in euclidean space.

\begin{figure}[h!]
\includegraphics[width=7cm]{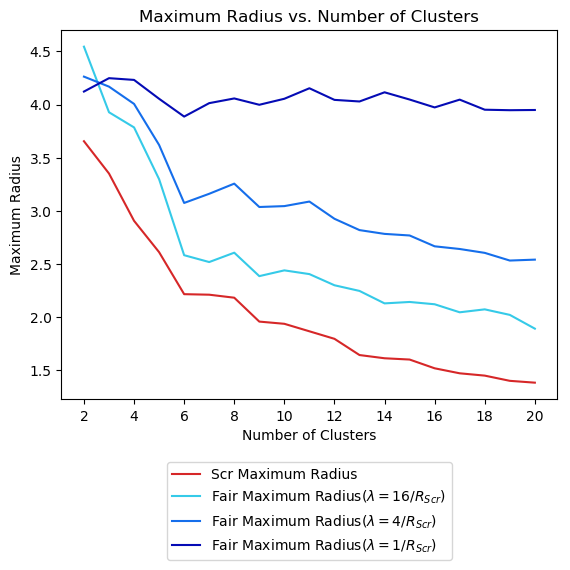}
  \caption{Comparison over different numbers of clusters, $k$, from 2 to 20 on the adult dataset. Here, we measure the maximum radius. In both cases, lower values on the y-axis represent more optimally compact clusters. We compare Scr to three versions of our algorithm parameterized by $\lambda$ of $16/R_{Scr}$, $4/R_{Scr}$, and $1/R_{Scr}$. We see that the more extreme fair algorithms (smaller $\lambda$ parameter) suffer a greater price of fairness, but this is constrained within the theoretical bounds shown in Section~\ref{sec:fairextension}.
  }
\label{plot:adult-radius}
\end{figure}

\subsection{Experimental design}

The design is similar to Section~\ref{sec:exp} with the following changes. To evaluate performance while changing the parameter $k$, we now study a single dataset, but vary the number of clusters, $k$, from 2 to 20. Given that we do not know the optimal radius for this data under different numbers of clusters, we use the actual radius instead of a ratio in Figure~\ref{plot:adult-radius}. In addition, we only consider one ``unfair'' algorithm, Scr, which gets closest to the optimal radius in practice.

\subsection{Experimental results}

Figure~\ref{plot:adult-fair} shows that the fairer algorithms are more fair as expected. However, we also see that as we scale the parameter toward greater fairness, the fairness level becomes more consistent and robust to different values of $k$. Figure~\ref{plot:adult-radius} illustrates the price of fairness we pay in terms of the maximum radius of any cluster. In all plots, we see predictably strange behavior at the extreme low values of $k$ (e.g., when $k = 2$, the maximum number of clusters a community can be fractured into is at most 2).

\section{Conclusion and future directions}

We introduced and motivated the concepts of pairwise fairness and community preservation to the $k$-center clustering problem. To explore the practicality of such constraints, we designed a randomized algorithm that can be combined with existing $k$-center algorithms or heuristics to ensure fairness at the expense of the objective value. We validated our algorithm both theoretically and experimentally.

In terms of future work, there are several open questions around how these new fairness concepts can be combined with other constraints or objectives including other definitions of fairness. For the $k$-center problem itself, it is unknown whether our bounds on fairness or the objective function can be improved. Further, one could ask if these fairness properties can be extended to variants of $k$-center such as capacitated $k$-center which is well-motivated by many real world applications. Other natural constraints to combine with include other notions of fairness or linkage constraints as seen in semi-supervised learning. We note that pairwise fairness and community preservation can be directly at odds with group fairness (e.g. if points belonging to the same group tend to be close together in the metric space). Finding the trade-off between these fairness concepts is an open problem although it is not clear that many application contexts would require both at the same time. Finally, these definitions could be extended to other common objectives such as k-median and k-means. Our algorithm targets $\alpha$ and $\beta$ which are functions of the unfair radius $R$, a natural parameter given the $k$-center objective.  However, for $k$-median, we may instead use the average distance from points to centers. While it is easy to see how our fairness definitions could apply to other objectives, our algorithm does not extend to these objectives.

\section*{Acknowledgements}

The authors wish to thank the anonymous reviewers for helpful feedback in improving the paper. Brian Brubach is supported in part by NSF awards CCF-1422569 and CCF-1749864, and by research awards from Adobe. Darshan Chakrabarti was supported in part via an REU grant, NSF CCF-1852352.  John Dickerson was supported in part by NSF CAREER Award IIS \#1846237, DARPA SI3-CMD Award S4761, and a Google Faculty Research Award. Samir Khuller is supported by an Amazon Research Award and an Adobe Award. Aravind Srinivasan is supported in part by NSF awards CCF-1422569, CCF-1749864, and CCF-1918749 as well as research awards from Adobe, Amazon, and Google. Leonidas Tsepenekas is supported in part by NSF awards CCF-1749864 and CCF-1918749, and by research awards from Amazon and Google.

\bibliography{refs}

\begin{thebibliography}{58}
\providecommand{\natexlab}[1]{#1}
\providecommand{\url}[1]{\texttt{#1}}
\expandafter\ifx\csname urlstyle\endcsname\relax
  \providecommand{\doi}[1]{doi: #1}\else
  \providecommand{\doi}{doi: \begingroup \urlstyle{rm}\Url}\fi

\bibitem[Cas(Rucho v. Common Cause, No. 18-422, 588 U.S. \_\_\_
  {(2019)})]{CaseRucho2019}
Rucho v. Common Cause, No. 18-422, 588 U.S. \_\_\_ {(2019)}.

\bibitem[Cas(Thornburg v. Gingles, No. 83-1968, 478 U.S. 30
  {(1986)})]{CaseThornburg1986}
Thornburg v. Gingles, No. 83-1968, 478 U.S. 30 {(1986)}.

\bibitem[Altman(1998)]{Altman1998}
Altman, M.
\newblock Modeling the effect of mandatory district compactness on partisan
  gerrymanders.
\newblock \emph{Political Geography}, 1998.

\bibitem[Altman \& McDonald(2010)Altman and McDonald]{Altman2010}
Altman, M. and McDonald, M.
\newblock The promise and perils of computers in redistricting.
\newblock \emph{Duke Journal of Constitutional Law and Public Policy}, 2010.

\bibitem[Ashlagi \& Shi(2014)Ashlagi and Shi]{Ashlagi2014}
Ashlagi, I. and Shi, P.
\newblock Improving community cohesion in school choice via correlated-lottery
  implementation.
\newblock \emph{Operations Research}, 2014.

\bibitem[Backurs et~al.(2019)Backurs, Indyk, Onak, Schieber, Vakilian, and
  Wagner]{Backurs2019}
Backurs, A., Indyk, P., Onak, K., Schieber, B., Vakilian, A., and Wagner, T.
\newblock Scalable fair clustering.
\newblock In \emph{International Conference on Machine Learning (ICML)}, pp.\
  405--413, 2019.

\bibitem[Bartal(1996)]{Bartal1996}
Bartal, Y.
\newblock Probabilistic approximation of metric spaces and its algorithmic
  applications.
\newblock In \emph{Annual Symposium on Foundations of Computer Science (FOCS)},
  1996.

\bibitem[Beasley(1985)]{Beasley1985}
Beasley, J.
\newblock A note on solving large p-median problems.
\newblock \emph{European Journal of Operational Research}, 1985.

\bibitem[Beasley(1990)]{OR-Lib1990}
Beasley, J.~E.
\newblock {OR-Library}: Distributing test problems by electronic mail.
\newblock \emph{The Journal of the Operational Research Society}, 1990.

\bibitem[Bera et~al.(2019)Bera, Chakrabarty, and Negahbani]{Bera19:Fair}
Bera, S.~K., Chakrabarty, D., and Negahbani, M.
\newblock Fair algorithms for clustering.
\newblock In \emph{Neural Information Processing Systems (NeurIPS)}, 2019.

\bibitem[Bercea et~al.(2019)Bercea, Gro{\ss}, Khuller, Kumar, R{\"o}sner,
  Schmidt, and Schmidt]{Bercea2018}
Bercea, I.~O., Gro{\ss}, M., Khuller, S., Kumar, A., R{\"o}sner, C., Schmidt,
  D.~R., and Schmidt, M.
\newblock {On the Cost of Essentially Fair Clusterings}.
\newblock In \emph{Approximation, Randomization, and Combinatorial
  Optimization. Algorithms and Techniques (APPROX/RANDOM 2019)}, 2019.

\bibitem[Brubach et~al.(2020)Brubach, Srinivasan, and
  Zhao]{Brubach2020Meddling}
Brubach, B., Srinivasan, A., and Zhao, S.
\newblock Meddling metrics: the effects of measuring and constraining partisan
  gerrymandering on voter incentives.
\newblock In \emph{ACM Conference on Economics and Computation (EC)}, 2020.

\bibitem[Chakrabarty et~al.(2016)Chakrabarty, Goyal, and
  Krishnaswamy]{Chakrabarty2016}
Chakrabarty, D., Goyal, P., and Krishnaswamy, R.
\newblock The non-uniform k-center problem.
\newblock In \emph{International Conference on Automata, Languages, and
  Programming (ICALP)}, 2016.

\bibitem[Charikar et~al.(1997)Charikar, Chekuri, Feder, and
  Motwani]{Charikar1997}
Charikar, M., Chekuri, C., Feder, T., and Motwani, R.
\newblock Incremental clustering and dynamic information retrieval.
\newblock In \emph{Annual Symposium on Theory of Computing (STOC)}, 1997.

\bibitem[Charikar et~al.(2003)Charikar, O'Callaghan, and
  Panigrahy]{Charikar2003}
Charikar, M., O'Callaghan, L., and Panigrahy, R.
\newblock Better streaming algorithms for clustering problems.
\newblock In \emph{Annual Symposium on Theory of Computing (STOC)}, 2003.

\bibitem[Chierichetti et~al.(2017)Chierichetti, Kumar, Lattanzi, and
  Vassilvitskii]{Chierichetti2017}
Chierichetti, F., Kumar, R., Lattanzi, S., and Vassilvitskii, S.
\newblock Fair clustering through fairlets.
\newblock In \emph{Neural Information Processing Systems (NeurIPS)}. 2017.

\bibitem[Cohen-Addad et~al.(2018)Cohen-Addad, Klein, and
  Young]{Cohen-Addad2018}
Cohen-Addad, V., Klein, P.~N., and Young, N.~E.
\newblock Balanced centroidal power diagrams for redistricting.
\newblock In \emph{ACM SIGSPATIAL International Conference on Advances in
  Geographic Information Systems}, 2018.

\bibitem[Cutting et~al.(1992)Cutting, Karger, Pedersen, and Tukey]{Cutting1992}
Cutting, D.~R., Karger, D.~R., Pedersen, J.~O., and Tukey, J.~W.
\newblock Scatter/gather: A cluster-based approach to browsing large document
  collections.
\newblock In \emph{International ACM SIGIR Conference on Research and
  Development in Information Retrieval}, 1992.

\bibitem[Daskin(2000)]{Daskin2000}
Daskin, M.
\newblock A new approach to solving the vertex p-center problem to optimality:
  Algorithm and computational results.
\newblock \emph{Communications of the Operations Research Society of Japan},
  2000.

\bibitem[Datar et~al.(2004)Datar, Immorlica, Indyk, and Mirrokni]{Datar2004LSH}
Datar, M., Immorlica, N., Indyk, P., and Mirrokni, V.~S.
\newblock Locality-sensitive hashing scheme based on p-stable distributions.
\newblock In \emph{Annual Symposium on Computational Geometry (SoCG)}, 2004.

\bibitem[Datta et~al.(2015)Datta, Tschantz, and Datta]{Datta15:Automated}
Datta, A., Tschantz, M.~C., and Datta, A.
\newblock Automated experiments on ad privacy settings.
\newblock \emph{Proceedings on Privacy Enhancing Technologies}, 2015.

\bibitem[Dwork et~al.(2012)Dwork, Hardt, Pitassi, Reingold, and
  Zemel]{Dwork2012}
Dwork, C., Hardt, M., Pitassi, T., Reingold, O., and Zemel, R.
\newblock Fairness through awareness.
\newblock In \emph{Innovations in Theoretical Computer Science Conference
  (ITCS)}, 2012.

\bibitem[Elloumi et~al.(2004)Elloumi, Labb\'e, and Pochet]{Elloumi2004}
Elloumi, S., Labb\'e, M., and Pochet, Y.
\newblock A new formulation and resolution method for the p-center problem.
\newblock \emph{INFORMS Journal on Computing}, 2004.

\bibitem[Fakcharoenphol et~al.(2003)Fakcharoenphol, Rao, and Talwar]{FRT2003}
Fakcharoenphol, J., Rao, S., and Talwar, K.
\newblock A tight bound on approximating arbitrary metrics by tree metrics.
\newblock In \emph{Annual Symposium on Theory of Computing (STOC)}, 2003.

\bibitem[Fernandes et~al.(2018)Fernandes, de~Paula, and Pedrosa]{Fernandes2018}
Fernandes, C.~G., de~Paula, S.~P., and Pedrosa, L. L.~C.
\newblock Improved approximation algorithms for capacitated fault-tolerant
  k-center.
\newblock \emph{Algorithmica}, 2018.

\bibitem[Fryer \& Holden(2011)Fryer and Holden]{Fryer2011}
Fryer, R.~G. and Holden, R.
\newblock Measuring the compactness of political districting plans.
\newblock \emph{The Journal of Law and Economics}, 2011.

\bibitem[Garcia-Diaz et~al.(2017)Garcia-Diaz, Sanchez-Hernandez,
  Menchaca-Mendez, and Menchaca-Mendez]{Garcia-Diaz2017}
Garcia-Diaz, J., Sanchez-Hernandez, J., Menchaca-Mendez, R., and
  Menchaca-Mendez, R.
\newblock When a worse approximation factor gives better performance: a
  3-approximation algorithm for the vertex k-center problem.
\newblock \emph{Journal of Heuristics}, 2017.

\bibitem[Ge et~al.(2008)Ge, Ester, Gao, Hu, Bhattacharya, and
  Ben-Moshe]{Ge2008}
Ge, R., Ester, M., Gao, B.~J., Hu, Z., Bhattacharya, B., and Ben-Moshe, B.
\newblock Joint cluster analysis of attribute data and relationship data: The
  connected k-center problem, algorithms and applications.
\newblock \emph{ACM Trans. Knowl. Discov. Data}, 2008.

\bibitem[Gionis et~al.(1999)Gionis, Indyk, and Motwani]{Gionis1999LSH}
Gionis, A., Indyk, P., and Motwani, R.
\newblock Similarity search in high dimensions via hashing.
\newblock In \emph{International Conference on Very Large Data Bases (VLDB)},
  1999.

\bibitem[Gonzalez(1985)]{Gonzalez1985}
Gonzalez, T.~F.
\newblock Clustering to minimize the maximum intercluster distance.
\newblock \emph{Theoretical Computer Science}, 1985.

\bibitem[Harris et~al.(2019)Harris, Li, Pensyl, Srinivasan, and
  Trinh]{Harris2019}
Harris, D.~G., Li, S., Pensyl, T., Srinivasan, A., and Trinh, K.
\newblock Approximation algorithms for stochastic clustering.
\newblock \emph{Journal of Machine Learning Research}, 2019.

\bibitem[Hochbaum \& Maass(1985)Hochbaum and Maass]{Hochbaum1985Grids}
Hochbaum, D.~S. and Maass, W.
\newblock Approximation schemes for covering and packing problems in image
  processing and vlsi.
\newblock \emph{J. ACM}, 1985.

\bibitem[Hochbaum \& Shmoys(1985)Hochbaum and Shmoys]{Hochbaum1985}
Hochbaum, D.~S. and Shmoys, D.~B.
\newblock A best possible heuristic for the k-center problem.
\newblock \emph{Math. Oper. Res.}, 1985.

\bibitem[Hochbaum \& Shmoys(1986)Hochbaum and Shmoys]{Hochbaum1986}
Hochbaum, D.~S. and Shmoys, D.~B.
\newblock A unified approach to approximation algorithms for bottleneck
  problems.
\newblock \emph{J. ACM}, 1986.

\bibitem[Indyk \& Motwani(1998)Indyk and Motwani]{Indyk1998LSH}
Indyk, P. and Motwani, R.
\newblock Approximate nearest neighbors: Towards removing the curse of
  dimensionality.
\newblock In \emph{Annual Symposium on Theory of Computing (STOC)}, 1998.

\bibitem[Jung et~al.(2020)Jung, Kannan, and Lutz]{Jung2020}
Jung, C., Kannan, S., and Lutz, N.
\newblock {Service in Your Neighborhood: Fairness in Center Location}.
\newblock In \emph{Symposium on Foundations of Responsible Computing (FORC
  2020)}, 2020.

\bibitem[Khuller \& Sussmann(1996)Khuller and Sussmann]{Khuller1996}
Khuller, S. and Sussmann, Y.~J.
\newblock The capacitated k-center problem.
\newblock In \emph{European Symposium on Algorithms}, 1996.

\bibitem[Khuller et~al.(2000)Khuller, Pless, and Sussmann]{Khuller2000}
Khuller, S., Pless, R., and Sussmann, Y.~J.
\newblock Fault tolerant k-center problems.
\newblock \emph{Theoretical Computer Science}, 2000.

\bibitem[Kleindessner et~al.(2019{\natexlab{a}})Kleindessner, Awasthi, and
  Morgenstern]{Kleindessner19:Fair}
Kleindessner, M., Awasthi, P., and Morgenstern, J.
\newblock Fair k-center clustering for data summarization.
\newblock In \emph{International Conference on Machine Learning (ICML)},
  2019{\natexlab{a}}.

\bibitem[Kleindessner et~al.(2019{\natexlab{b}})Kleindessner, Samadi, Awasthi,
  and Morgenstern]{Kleindessner19:Guarantees}
Kleindessner, M., Samadi, S., Awasthi, P., and Morgenstern, J.
\newblock Guarantees for spectral clustering with fairness constraints.
\newblock In \emph{International Conference on Machine Learning (ICML)},
  2019{\natexlab{b}}.

\bibitem[Kohavi \& Becker(1996)Kohavi and Becker]{Kohavi1996}
Kohavi, R. and Becker, B.
\newblock {UCI} machine learning repository, 1996.
\newblock URL \url{https://archive.ics.uci.edu/ml/datasets/Adult}.

\bibitem[Lim et~al.(2004)Lim, Rodrigues, Wang, and Xu]{Lim2004}
Lim, A., Rodrigues, B., Wang, F., and Xu, Z.
\newblock k-center problems with minimum coverage.
\newblock In \emph{Computing and Combinatorics}, 2004.

\bibitem[Linial \& Saks(1993)Linial and Saks]{Linial1993}
Linial, N. and Saks, M.
\newblock Low diameter graph decompositions.
\newblock \emph{Combinatorica}, 1993.

\bibitem[Liu et~al.(2016)Liu, Cho, and Wang]{Liu2016}
Liu, Y.~Y., Cho, W. K.~T., and Wang, S.
\newblock {PEAR}: a massively parallel evolutionary computation approach for
  political redistricting optimization and analysis.
\newblock \emph{Swarm and Evolutionary Computation}, 2016.

\bibitem[Mahabadi \& Vakilian(2020)Mahabadi and Vakilian]{Mahabadi2020}
Mahabadi, S. and Vakilian, A.
\newblock (individual) fairness for $k$-clustering.
\newblock In \emph{International Conference on Machine Learning (ICML)}, 2020.

\bibitem[Malkomes et~al.(2015)Malkomes, Kusner, Chen, Weinberger, and
  Moseley]{Malkomes2015}
Malkomes, G., Kusner, M.~J., Chen, W., Weinberger, K.~Q., and Moseley, B.
\newblock Fast distributed k-center clustering with outliers on massive data.
\newblock In \emph{Neural Information Processing Systems (NeurIPS)}, 2015.

\bibitem[McCutchen \& Khuller(2008)McCutchen and Khuller]{McCutchen2008}
McCutchen, R.~M. and Khuller, S.
\newblock Streaming algorithms for k-center clustering with outliers and with
  anonymity.
\newblock In \emph{Approximation, Randomization and Combinatorial Optimization.
  Algorithms and Techniques}, 2008.

\bibitem[Menino(2012)]{Menino2012}
Menino, T.~M.
\newblock {The Honorable Mayor Thomas M. Menino state of the city address},
  January 2012.

\bibitem[Mihelič \& Robic(2005)Mihelič and Robic]{Mihelic2005}
Mihelič, J. and Robic, B.
\newblock Solving the k-center problem efficiently with a dominating set
  algorithm.
\newblock \emph{CIT}, 2005.

\bibitem[Mladenovic et~al.(2003)Mladenovic, Labb\'e, and
  Hansen]{Mladenovic2003}
Mladenovic, N., Labb\'e, M., and Hansen, P.
\newblock Solving the p-center problem with tabu search and variable
  neighborhood search.
\newblock \emph{Networks}, 2003.

\bibitem[R{\"o}sner \& Schmidt(2018)R{\"o}sner and Schmidt]{Rosner18:Privacy}
R{\"o}sner, C. and Schmidt, M.
\newblock Privacy preserving clustering with constraints.
\newblock \emph{International Conference on Automata, Languages, and
  Programming (ICALP)}, 2018.

\bibitem[Srinivasan et~al.(2010)Srinivasan, Galb{\'a}n, Johnson, Chenevert,
  Ross, and Mukherji]{Srinivasan10:Utility}
Srinivasan, A., Galb{\'a}n, C., Johnson, T., Chenevert, T., Ross, B., and
  Mukherji, S.
\newblock Utility of the k-means clustering algorithm in differentiating
  apparent diffusion coefficient values of benign and malignant neck
  pathologies.
\newblock \emph{American Journal of Neuroradiology}, 2010.

\bibitem[{Supreme Court of the United States}(2015)]{SCOTUS15:Texas}
{Supreme Court of the United States}.
\newblock 13-1371 -- {T}exas {D}epartment of {H}ousing and {C}ommunity
  {A}ffairs v.\ {T}he {I}nclusive {C}ommunities {P}roject, {I}nc., January
  2015.

\bibitem[Thorup(2001)]{Thorup2001}
Thorup, M.
\newblock Quick k-median, k-center, and facility location for sparse graphs.
\newblock In \emph{Automata, Languages and Programming}, 2001.

\bibitem[{United States Census Bureau, Population Division}(2016, accessed May,
  2018)]{AustinSize2016}
{United States Census Bureau, Population Division}.
\newblock \emph{{``American FactFinder -- Results''}}, 2016, accessed May,
  2018.
\newblock URL
  \url{https://factfinder.census.gov/faces/tableservices/jsf/pages/productview.xhtml?src=bkmk}.

\bibitem[{United States Senate}(1991)]{CivilRights1991}
{United States Senate}.
\newblock S.\ 1745 -- 102nd {C}ongress: {C}ivil {R}ights {A}ct of 199, 1991.
\newblock \url{https://www.govtrack.us/congress/bills/102/s1745}.

\bibitem[Zamir et~al.(1997)Zamir, Etzioni, Madani, and Karp]{Zamir1997}
Zamir, O., Etzioni, O., Madani, O., and Karp, R.~M.
\newblock Fast and intuitive clustering of web documents.
\newblock In \emph{International Conference on Knowledge Discovery and Data
  Mining (KDD)}, 1997.

\bibitem[Zemel et~al.(2013)Zemel, Wu, Swersky, Pitassi, and Dwork]{Zemel2013}
Zemel, R., Wu, Y., Swersky, K., Pitassi, T., and Dwork, C.
\newblock Learning fair representations.
\newblock In \emph{International Conference on Machine Learning (ICML)}, 2013.

\end{thebibliography}
\bibliographystyle{icml2020}

%
%
%

\end{document}